\documentclass[a4paper]{article}

\usepackage{times}
\usepackage{graphicx} % more modern

\usepackage[affil-it]{authblk}
\usepackage{natbib}
\usepackage{amsfonts}
\usepackage{fancyhdr}
\usepackage{amssymb, amsmath}
\usepackage{amsthm}
\usepackage{mathtools}
\usepackage{verbatim}
\usepackage{tikz}
\usetikzlibrary{positioning}

\newcommand{\hypothesis}{h}
\newcommand{\predfun}{f}
\newcommand{\hypspace}{\mathcal{H}}
\newcommand{\funset}{\mathcal{F}}

\newcommand{\kernelf}{K} % adjusted
\newcommand{\lossfunction}{L}

\newcommand{\regparam}{\lambda}
\newcommand{\queryspace}{\mathcal{Q}}
\newcommand{\pointspace}{\mathcal{P}}
\newcommand{\anyspace}{\mathcal{X}}

\newcommand{\node}{v}

\newcommand{\query}{q}

\newcommand{\bm}[1]{\mathbf{#1}}

\newcommand{\idmatrix}{\bm{I}}

\newcommand{\anymatrix}{\bm{M}}
\newcommand{\othermatrix}{\bm{N}}

\newcommand{\tsize}{m}

\newcommand{\tset}{T}

\newcommand{\evecmatrix}{\bm{V}}
\newcommand{\evalmatrix}{\bm{\Lambda}}
\newcommand{\shufflem}{\bm{P}}

\newcommand{\symm}{\bm{S}}
\newcommand{\asymm}{\bm{A}}

\newcommand{\asymfun}{t}

\newcommand{\defect}{I}
\newcommand{\effdim}{D}
\newcommand{\koper}{\bm{T}}
\newcommand{\krausoper}{\bm{E}}
\newcommand{\lmutorkhs}{\bm{U}}
\newcommand{\arbfun}{r}
\newcommand{\approxfun}{u}

\newtheorem{theorem}{Theorem}[section]
\newtheorem{lemma}[theorem]{Lemma}
\newtheorem{proposition}[theorem]{Proposition}
\newtheorem{corollary}[theorem]{Corollary}
\newtheorem{definition}[theorem]{Definition}
\newtheorem{remark}[theorem]{Remark}
\begin{document}

\title{Spectral Analysis of Symmetric and Anti-Symmetric Pairwise Kernels}
\author[1]{Tapio Pahikkala}
\author[1]{Markus Viljanen}
\author[1]{Antti Airola}
\author[2]{Willem Waegeman}
\affil[1]{Department of Information Technology, University of Turku, Joukahaisenkatu 3-5 B, FIN-20520, Turku, Finland, firstname.surname@utu.fi}
\affil[2]{Department of Mathematical Modelling, Statistics and Bioinformatics, Ghent University, Coupure links 653, B-9000 Ghent, Belgium, firstname.surname@UGent.be}
\maketitle

\begin{abstract}

We consider the problem of learning regression functions from pairwise data when there exists
prior knowledge that the relation to be learned is symmetric or anti-symmetric. Such prior knowledge
is commonly enforced by symmetrizing or anti-symmetrizing pairwise kernel functions. Through spectral
analysis, we show that these transformations reduce the kernel's effective dimension. Further,
we provide an analysis of the approximation properties of the resulting kernels, and bound the
regularization bias of the kernels in terms of the corresponding bias of the original kernel.

\end{abstract}

\section{Introduction}

%Learning symmetric and anti-symmetric functions are special cases of pairwise learning tasks, which can be further divided into more specific problems. For example, one of the most widely studied learning problems involving anti-symmetric functions is that of learning ordering relations, that is, given two objects, the aim is to correctly predict which is better. Typical tasks involving symmetric functions are, for example, the task of determining whether two objects belong to the same class, or what is the distance between them, etc.

Many real-world phenomena can be described in tems of pairwise relationships between entities.  When learning pairwise relations, symmetry and anti-symmetry are two types of prior knowledge constraints that commonly appear when both of the objects in a pair belong to the same domain. A typical example of an application where relationships are often assumed to be symmetric is the prediction of protein-protein interactions: if protein A interacts with protein B, then conversely it also holds that B interacts with A. Typical example of an anti-symmetric relation would be a preference relation: if A is preferred over B, then conversely B is not preferred over A. Commonly used symmetric pairwise kernels include the symmetrized Kronecker \citep{Benhur2005} and Cartesian \citep{kashima2009pairwise},  
as well as the metric learning \citep{vert2007new} kernels. Such kernels are analyzed in more detail by \citet{brunner2012pairwise}.
Typical examples of anti-symmetric kernels are the transitive kernel of \citep{Herbrich2000} used for learning to rank, and the anti-symmetric Kronecker product kernel \citep{pahikkala2010reciprocalkm} for learning intransitive preference relations.

Kernel-based learning algorithms are some of the most successful learning methods in practise and they also enjoy strong theoretical properties. It is well known in the machine learning literature that the eigenvalues and eigenfunctions of the integral operator of the kernel play a central role in obtaining error estimates in learning theory. One of the most intensively studied quantities depending on the eigenvalues is the so-called \emph{effective dimension} of the kernel, which has since its introduction by \citet{Zhang2002effective} been used by several other authors \citep{Mendelson2003kernelclasses,Caponnetto2007optimalrates}. For a recent summary of these results, see \citet{HsuK014randomdesign} and references therein. Therefore, the determination of the operator's eigensystem is important in its own right. Another important tool for analysis is the theory of universal kernels pioneered by \citet{Steinwart2002consistency}, which indicates that if a kernel has the so-called unversality property, the corresponding hypothesis space can approximate any continuous function arbitrarily well.

%What this means in practise is that a learning algorithm is able to learn the concept arbitrarily well provided that it is given a large and representative enough training set. %On the other hand, there exists a wide literature about the generalization properties of kernel methods depending on the properties of the kernel functions.

Intuitively it seems plausible that enforcing prior knowledge about symmetry or anti-symmetry should result in better generalization, and many promising experimental results have been obtained in the literature (see previous references). However, thus far rigorous theoretical analysisis of the effects that enforcing these properties on the kernel function has on learning has been missing in the literature. As a step towards this direction \citet{waegeman2012learninggraded} have shown that when symmetrizing or anti-symmetrizing pairwise kernels that are formed by taking the Kronecker product of two universal kernels, the resulting kernel allows approximating arbitrarily well any symmetric or anti-symmetric continuous function. While these results show that symmetrization or anti-symmetrization does not sacrifice expressive power needed for learning, the results concern only Kronecker product kernels, and do not provide any guarantees that learning would be more efficient with the transformed kernels.

%In this work, we establish the following main results:
Following are the main contributions and results of our paper:
\begin{itemize}
\item The effective dimension of both the symmetrized and anti-symmetrized versions of a pairwise kernel are smaller than that of the original pairwise kernel (see Theorem~\ref{effdimtheorem}).
\item The approximation properties of the symmetric and anti-symmetric kernels are analysed (see Theorem~\ref{generalantisymmetrictheorem}).
\item We bound the regularization bias of the symmetric and anti-symmetric kernels in terms of the regularization bias of the original kernel (see Theorem~\ref{regerrpropo}).
\end{itemize}

%In addition to the main results concerning the most general types of symmetric and anti-symmetric kernels, we present corollaries concerning some of the most important special cases, such as pairwise ordering relations and conditional ranking.
%, learning to match, learning with Kronecker kernels, et cetera.

%\section{Analysis of Symmetric and Anti-Symmetric RKHS}
%\section{Analysis of Pairwise Kernels}
\section{Preliminaries}

%\subsection{General Kernel Properties}

%\begin{definition}
% The Kronecker product of two matrices $\anymatrix$ and $\othermatrix$ is defined as
% \begin{eqnarray*}
% \anymatrix\otimes\othermatrix=\left(
% \begin{array}{ccc}
% {\anymatrix}_{1,1}\othermatrix&\cdots&\anymatrix_{1,n}\othermatrix\\
% \vdots&\ddots&\vdots\\
% {\anymatrix}_{m,1}\othermatrix&\cdots&\anymatrix_{m,n}\othermatrix
% \end{array}
% \right),
% \end{eqnarray*}
% \end{definition}

\begin{definition}[Kernel function]
For any set $\anyspace$, the function $\kernelf$ is a kernel if it can be written as the following type of an inner product:
\[
\kernelf(x,\overline{x})=\langle\Phi(x),\Phi(\overline{x})\rangle\;,
\]
where
\[
\Phi:\anyspace\rightarrow\mathcal{H}_\Phi
\]
is a mapping from $\anyspace$ to a Hilbert space $\mathcal{H}_\Phi$, popularly called the feature space in the literature. Conversely, any kernel can be written as the above type of an inner product. However, neither the feature mapping nor the feature space are unique.
\end{definition}
To simplify the forthcoming considerations, we make a couple of extra assumptions of the input space and kernels. Namely, we assume that the input space $\anyspace$ is compact (e.g. closed and bounded) and the kernel functions considered in this article are continuous. Let $\mu$ be a probability distribution over $\anyspace$ generating the data. We also assume that $\mu$ is a probability density with respect to a Lebesque measure (e.g. we can write $\int_\anyspace \hypothesis(x)d\mu(x)=\int_\anyspace \hypothesis(x)\mu(x)dx$ for any function $\hypothesis$).

We make use of the Hilbert space $L^2(\anyspace, \mu)$ of square integrable functions on $(\anyspace,\mu)$ with the inner product $\langle\hypothesis,g\rangle_{L^2(\anyspace, \mu)}=\int_x\hypothesis(x)g(x)d\mu(x)$. The elements of the space $L^2(\anyspace, \mu)$ are equivalence classes of functions rather that individual functions but this technical detail has no effect on the considerations below.

\begin{definition}[\citep{aronszajn1950}]
For each real-valued kernel $\kernelf$ and an input space $\anyspace$, there exists a unique Hilbert space $\hypspace(\kernelf)$ known as the reproducing kernel Hilbert space (RKHS):
%\begin{equation*}
%\hypspace(\kernelf)=\overline{\left\{ \hypothesis \in \mathbb{R}^\anyspace  \left\arrowvert
%\hypothesis(\cdot)=\sum_{i=1}^\tsize\beta_i\kernelf_{x_i},\phantom{w}x_i\in\anyspace\right.\right\}}\,,
%\end{equation*}
%where $\kernelf_{x}\in\hypspace(\kernelf)$ are functions such that $\kernelf_{x}(\overline{x})=\kernelf(x,\overline{x})$
\begin{enumerate}
  \item $\kernelf_{x}\in\hypspace(\kernelf)\phantom{W}\forall x\in\anyspace$, where
\[
\kernelf_{x}:\anyspace\rightarrow\mathbb{R}
\]
 are functions such that $\kernelf_{x}(\overline{x})=\kernelf(x,\overline{x})$
  \item $\textnormal{span}(\{\kernelf_{x}\}_{x\in\anyspace})$ is dense in $\hypspace(\kernelf)$
%\begin{enumerate}
%  \item $\kernelf(\cdot,x) \in \hypspace(\kernelf) \phantom{W} \forall x \in \anyspace$ %, where $K x (t) = K(x, t)$
%  
\item The inner product $\langle\cdot,\cdot\rangle_{\hypspace(\kernelf)}$ associated with $\hypspace(\kernelf)$ satisfies:
\[
f(x) = \langle f, \kernelf_{x}\rangle \phantom{W} \forall f \in \hypspace(\kernelf),\phantom{w}x\in\anyspace
\]
which is known as the reproducing property. In particular,
\[
\kernelf(x,\overline{x})=\langle \kernelf_{x}, \kernelf_{\overline{x}}\rangle \phantom{W} \forall x,\overline{x}\in\anyspace\;.
\]
\end{enumerate}
%\item $\textnormal{span}{\kernelf(\cdot,x)}$ is dense in $\hypspace(\kernelf)$
%\end{enumerate}
\end{definition}
In the literature, the mapping:
\[
\Phi_\kernelf:x\rightarrow\kernelf_x\in\hypspace(\kernelf)
\]
is often referred to as the canonical feature map of the kernel.

 % In fact, we an always assume that $\mathcal{I}=\hypspace(\kernelf)$ (e.g. the feature space is equal to the RKHS of $\kernelf$) \citep{Micchelli2006universal}, which we do to simplify the forthcoming considerations.

%Following the standard notations for kernel methods, we formulate our learning problem as the selection of a suitable function $\hypothesis\in\hypspace$, with $\hypspace$ a certain hypothesis space, in particular a reproducing kernel Hilbert space (RKHS).

\begin{definition}[Integral operator of a kernel]\label{intopdef}
The probability distribution $\mu$ over $\anyspace$ yields a linear operator
\[
\lmutorkhs_\kernelf : L^2(\anyspace, \mu) \rightarrow \mathcal{H}(\kernelf)
\]
defined as
\[
\lmutorkhs_\kernelf\hypothesis=\int_\anyspace \kernelf_x\hypothesis(x)d\mu(x)\;.
\]
The adjoint of this operator is the inclusion $\lmutorkhs_\kernelf^*:\mathcal{H}(\kernelf)\xhookrightarrow{} L^2(\anyspace, \mu)$, that is,
\begin{align}
\langle\lmutorkhs_\kernelf\hypothesis,g\rangle_{\mathcal{H}(\kernelf)}=\langle\hypothesis,\lmutorkhs_\kernelf^*  g\rangle_{L^2(\anyspace, \mu)}\;.
\end{align}
Note the RKHS norm on the left hand side, determined by the reproducing property, being changed to the $L^2(\anyspace, \mu)$ norm on the right. The composition of $\lmutorkhs_\kernelf$ with its adjoint is the operator:
%for all $\hypothesis\in L^2(\anyspace, \mu)$.
\[
\koper_\kernelf : L^2(\anyspace, \mu) \rightarrow L^2(\anyspace, \mu)\;.
\]
%defined as
%\[
%\koper_\kernelf\hypothesis=\int_\anyspace \kernelf(\cdot, x)\hypothesis(x)d\mu(x)
%\]
for all $\hypothesis\in L^2(\anyspace, \mu)$. This decomposition is illustrated in the following commutative diagram:
\begin{center}
\begin{tikzpicture}
    \node (L) at (0,0) {$L^2(\anyspace, \mu)$};
    \node[below right=of L] (H) {$\mathcal{H}(\kernelf)$};
    \node[above right=of H] (L2) {$L^2(\anyspace, \mu)$};
    \draw[->] (L)--(H) node [midway,below] {$\lmutorkhs_\kernelf$};
    \draw[->] (L)--(L2) node [midway, above] {$\koper_\kernelf$};
    \draw[->] (H)--(L2) node [midway,below] {$\lmutorkhs_\kernelf^*$};
\end{tikzpicture}
\end{center}
\end{definition}
The operator $\koper_\kernelf$ can be shown to be continuous, self-adjoint and Hilbert-Schmidt, the last property indicating that its eigenvalues are square-summable, which is characterized below in more detail. We next recollect some classical results from functional analysis required in the forthcoming considerations.
\begin{theorem}[Spectral theorem for compact operators] Suppose $\mathcal{L}$ is a Hilbert space and $\koper:\mathcal{L}\rightarrow\mathcal{L}$ is compact and self-adjoint linear operator. Then, $\mathcal{L}$ has an orthonormal basis $\{\phi_i\}_i$ consisting of eigenvectors of $\koper$.
%Furthermore, for each $\lambda\neq 0, \dim \mathcal{T}_\lambda < \infty$, and for each $\epsilon>0,\left\{\lambda\mid\arrowvert\lambda\arrowvert\geq\epsilon\textnormal{ and }\dim \mathcal{T}_\lambda >0\right\}$ is finite.
\end{theorem}
To compress the forthcoming notation and to take advantage the machinery of operator algebra, we use the following expression for the eigen decomposition of the integral operators: 
\[
\koper=\bm{V}\bm{\Lambda}\bm{V}^*\;,
\]
where $\bm{V}:e_i\mapsto\phi_i$ and $\bm{\Lambda}:e_i\mapsto\lambda_ie_i$, with $e_i$ being the standard basis vectors of $l^2$.
% consists of the eigenfunctions and $\bm{\Lambda}$ of the corresponding eigenvalues of $\koper$. One can think $\bm{\Lambda}$ as an element of the $l^2$ space and the unitary operator $\bm{V}$ as a mapping from $l^2$ to $\mathcal{L}$, but we mainly omit the technical details, since the operator algebra is mostly analogous to that of the usual matrix algebra.

%Consider an Hilbert space $H$ (e.g. the finite-dimensional Cn), and
%\[
% \mathcal{F}\subseteq\operatorname{Hom}(H,H)
%\]
%a commutative set of operators. If all the operators in $\mathcal{F}$ are compact then the operators can be simultaneously (unitarily) diagonalised.
For the integral operators of continuous kernels on compact domains, we have the following result known as Mercer's theorem:
\begin{theorem}[Mercer 1909]
Suppose $\kernelf$ is a continuous symmetric non-negative definite kernel. Then there is an orthonormal basis $\{\phi_i\}_i$ of $L^2(\anyspace)$ consisting of eigenfunctions of $T_\kernelf$ such that the corresponding sequence of eigenvalues $\{\lambda_i\}_i$ is nonnegative. The eigenfunctions corresponding to non-zero eigenvalues are continuous on $\anyspace$ and $\kernelf$ has the representation
\[
    \kernelf(x,\overline{x}) = \sum_{j\in\mathbb{N}} \lambda_j \, \phi_j(x) \, \phi_j(\overline{x}) 
\]
where the convergence is absolute and uniform.
\end{theorem}

The spectral theorem also yields the following corollary about commuting compact and self-adjoint operators sharing the same eigen system (see e.g. \citet{zimmer1990essential}):
\begin{corollary}\label{commutationcoro}
%3.2.5. Let $\{T_\alpha\}_{\alpha\in I}$ be a subset of $B(E)$ such that each $T_\alpha$ is compact, self-adjoint, and $T_\alpha T_\beta = T_\beta T_\alpha$ for all $\alpha,\beta\in I$. Then there is an orthonormal basis $\{e_j\}$ of $E$ such that $e_j$ an eigen vector for every $T_\alpha$.
Let $\mathcal{T}$ be a Hilbert space and let $\koper_1:\mathcal{L}\rightarrow\mathcal{L}$ and $\koper_2:\mathcal{L}\rightarrow\mathcal{L}$ be compact and self-adjoint operators, such that $\koper_1\koper_2=\koper_2\koper_1$. Then there is an orthonormal basis $\{\phi_j\}$ of $\mathcal{L}$ such that $\phi_j$ an eigenvector for both $\koper_1$ and $\koper_2$.
\end{corollary}

% \begin{theorem}[Mercer representation of RKHS]
% The RKHS of the kernel $\kernelf$ can be expressed as the Hilbert space:
% \[
% \hypspace=\left\{\sum_{j\in\mathbb{N}}a_j\sqrt{\lambda_j}\phi_j(\cdot), (a_i)\in l_2(\mathbb{N})\right\}
% \]
% equipped with the inner product
% \[
% \langle\hypothesis, g\rangle_\hypspace=\sum_{j\in\mathbb{N}}a_jb_j
% \]
% for all $\hypothesis=\sum_{j\in\mathbb{N}}a_j\sqrt{\lambda_j}\phi_j(\cdot)\in\hypspace$ and $g=\sum_{j\in\mathbb{N}}b_j\sqrt{\lambda_j}\phi_j(\cdot)\in\hypspace$. Further, the set $\left\{\sqrt{\lambda_j} \phi_j\right\}_{j=1}^\infty$ forms an orthonormal basis for $\hypspace$.
% \end{theorem}

Next, we define the concept of majorization for sequences of infinite lengths (see e.g. \citet{Li2013vnem} and references therein).
\begin{definition}[Majorization]\label{majordef}
Let $\bm{r}=(r_i)_{i=1}^\infty\in c_0^*$ and $\bm{s}=(s_i)_{i=1}^\infty\in c_0^*$ where $c_0^*$ is the positive cone of sequences decreasing monotonically to 0. We say that $\bm{s}$ majorizes $\bm{r}$, denoted as $\bm{r}\prec\bm{s}$ if
\[
\sum_{i=1}^\tsize r_i\leq\sum_{i=1}^\tsize s_i
%\textnormal{ for }\tsize=1,2,\ldots
\phantom{i}\forall \tsize\in\mathbb{N}
\textnormal{ and }\sum_{i=1}^\infty r_i=\sum_{i=1}^\infty s_i\;.
\]
In particular, for two trace class operators $\koper_1$ and $\koper_2$ on a Hilbert space, we say that $\koper_2\prec\koper_1$ if the sequence of eigenvalues of $\koper_1$ majorizes the sequence of eigenvalues of $\koper_2$.
\end{definition}
%\citet{Nielsen2001Majorization}
The next result is a recent generalization by \citet{Li2013vnem} of the classical Uhlmann's theorem for infinite dimensional Hilbert spaces. Before that, we also define the doubly-stochastic operations, which is also by \citet{Li2013vnem}:
\begin{definition}[Doubly-stochastic operation]
Let $\mathcal{T}(\mathcal{L})$ denote the (Banach) space of all trace class operators on a Hilbert space $\mathcal{L}$. We say that operation $\Gamma:\mathcal{T}(\mathcal{L})\rightarrow\mathcal{T}(\mathcal{L})$ is doubly-stochastic if it preserves trace (e.g. $\operatorname{trace}(\koper)=\operatorname{trace}(\Gamma(\koper))$), is unital indicating that $\idmatrix=\Gamma(\idmatrix)$ for the identity operator $\idmatrix$ on the Hilbert space, and there exists a sequence $\{\krausoper_i\}_{i=1}^\infty$ of compact operators on the Hilbert space $\mathcal{L}$, known in the literature as the Kraus operators, such that the operation can be written as
\begin{align}\label{krausform}
\Gamma(\koper)=\sum_{i=1}^\infty \krausoper_i\koper \krausoper_i^*\;.%\textrm{ with }\sum_{i=1}^\infty\bm{E}_i\bm{E}_i^*\leq\bm{I}
\end{align}
\end{definition}
\begin{theorem}[Uhlmann's theorem for infinite dimensional Hilbert spaces]
If $\koper_1$ and $\koper_2$ are trace-class operators on a Hilbert space, then $\koper_2\prec\koper_1$ iff there exists a doubly-stochastic operation $\Gamma$ such that $\koper_2=\Gamma(\koper_1)$.
\end{theorem}

\section{Pairwise Kernels}

%Systems of Identical Particles
%Robert B. Griffiths
%Version of 21 March 2011

Let us next define the family of pairwise kernels. Assume that the input space can be written as
\[
\anyspace=\pointspace^2
\]
where $\pointspace$ is a compact metric space. The kernels over $\pointspace^2$ can accordingly be written as the following types of inner products
\[
\kernelf(\node,\node',\overline{\node},\overline{\node}')=\langle\Phi(\node,\node'),\Phi(\overline{\node},\overline{\node}')\rangle\;,
\]
where $\node,\node',\overline{\node},\overline{\node}'\in\pointspace$ and $\Phi$ is a joint feature mapping over a pair of inputs, that is, $\Phi(\node,\node')$ is a feature space representation for an ordered pair $(\node,\node')$.

Next, we define certain specific types of pairwise kernels, starting from the permuted kernel:
\begin{definition}[Permuted pairwise kernel]\label{permutedkdef}
Let $\kernelf(\node,\node',\overline{\node},\overline{\node}')$ be an arbitrary kernel on $\pointspace^2$. Then, its permuted pairwise kernel is
\begin{align*}%\label{permutedkernel}
\kernelf^P(\node,\node',\overline{\node},\overline{\node}')=\kernelf(\node',\node,\overline{\node}',\overline{\node})\;.
\end{align*}
% whose feature mapping can be expressed as
% \[
% \kernelf^{P}_{(\overline{\node},\overline{\node}')}=\kernelf_{(\overline{\node}',\overline{\node})}
% \]
\end{definition}
An immediate step forward is to define the following type of kernels that are invariant to the permutations in the above defined sense:
\begin{definition}[Permutation invariant pairwise kernels]\label{strangekdef}
We say that a kernel \\$\kernelf^{PI}(\node,\node',\overline{\node},\overline{\node}')$ on $\pointspace^2$ is permutation invariant if it is equal to its permuted kernel, that is,
\[
\kernelf^{PI}(\node,\node',\overline{\node},\overline{\node}')=\kernelf^{PI}(\node',\node,\overline{\node}',\overline{\node})\;.
\]
A natural way to construct a permutation invariant kernel from a given pairwise kernel $\kernelf$ is to consider the projection from the set of all kernels to the set of permutation invariant kernels:
\begin{align*}%\label{strangekernel}
\kernelf^{PI}(\node,\node',\overline{\node},\overline{\node}')
=\frac{1}{2}\left(\kernelf(\node,\node',\overline{\node},\overline{\node}')
+\kernelf(\node',\node,\overline{\node}',\overline{\node})\right)\;.
\end{align*}
% A convenient way to express the feature mapping corresponding to this kernel is to stack the feature representations of $\kernelf$ for the original input $(\node,\node')$ and for the permuted input $(\node',\node)$:
% \begin{align*}
% \Phi^{PI}(\node,\node')=\frac{1}{\sqrt{2}}\left(\begin{array}{c}
% \Phi(\node,\node')\\
% \Phi(\node',\node)
% \end{array}\right)\;.
% \end{align*}
\end{definition}

%We assume that the measure $\mu$ is symmetric in the sense that $\int_\pointspace^2\hypothesis(\node,\node') d\mu(\node,\node')=\int_\pointspace^2\hypothesis(\node,\node') d\mu(\node',\node)\forall \hypothesis$.

%\[
%\asymm
%\]

%Let us define a permutation of arguments operator for the set of functions over $\pointspace^2$:
%\eq{
%\pi:\hypspace(\kernelf)&\rightarrow\hypspace(\kernelf)\\
%\hypothesis(\node,\node')&\mapsto\hypothesis(\node',\node)
%}
Our next step is to define the well-known symmetric pairwise kernels as well as their anti-symmetric counterparts:
\begin{definition}[Symmetric and anti-symmetric pairwise kernels]\label{symmasymmkdef}
We say that a kernel $\kernelf^{S}(\node,\node',\overline{\node},\overline{\node}')$ on $\pointspace^2$ is a symmetric pairwise kernel if
\[
\kernelf^{S}(\node,\node',\overline{\node},\overline{\node}')=\kernelf^{S}(\node',\node,\overline{\node},\overline{\node}')\;.
\]
Analogously, we say that a kernel $\kernelf^{A}(\node,\node',\overline{\node},\overline{\node}')$ on $\pointspace^2$ is an anti-symmetric pairwise kernel if
\[
\kernelf^{A}(\node,\node',\overline{\node},\overline{\node}')=-\kernelf^{A}(\node',\node,\overline{\node},\overline{\node}')\;.
\]
Similarly to the permutation invariance, one can construct symmetric and anti-symmetric kernels from an arbitrary kernel $\kernelf(\node,\node',\overline{\node},\overline{\node}')$ with the following projections:
\\$\kernelf^S(\node,\node',\overline{\node},\overline{\node}')=$
\begin{align*}%\label{skernel}
&\frac{1}{4}\Big(
\kernelf(\node,\node',\overline{\node},\overline{\node}')
+\kernelf(\node',\node,\overline{\node},\overline{\node}')
+\kernelf(\node,\node',\overline{\node}',\overline{\node})
+\kernelf(\node',\node,\overline{\node}',\overline{\node})
\Big)%\\
%&\langle\Phi(\node,\node')+\Phi(\node',\node),\Phi(\overline{\node},\overline{\node}')+\Phi(\overline{\node}',\overline{\node})\rangle
\end{align*}
and $\kernelf^A(\node,\node',\overline{\node},\overline{\node}')=$
\begin{align*}%\label{askernel}
&\frac{1}{4}\Big(
\kernelf(\node,\node',\overline{\node},\overline{\node}')
-\kernelf(\node',\node,\overline{\node},\overline{\node}')
-\kernelf(\node,\node',\overline{\node}',\overline{\node})
+\kernelf(\node',\node,\overline{\node}',\overline{\node})
\Big)\;,
\end{align*}
respectively.
% Moreover, the corresponding feature maps of the symmetric and anti-symmetric kernels can be expressed as:
% \begin{align*}
% \kernelf^{S}_{(\overline{\node},\overline{\node}')}=&\frac{1}{2}\left(\kernelf_{(\overline{\node},\overline{\node}')}+\kernelf_{(\overline{\node}',\overline{\node})}\right)\\
% \kernelf^{A}_{(\overline{\node},\overline{\node}')}=&\frac{1}{2}\left(\kernelf_{(\overline{\node},\overline{\node}')}-\kernelf_{(\overline{\node}',\overline{\node})}\right)\;,
% \end{align*}
% respectively.
\end{definition}
The following connection between the symmetric, anti-symmetric and permutation invariant kernels is immediate:
\begin{lemma}
Both the symmetric and anti-symmetric pairwise kernels are permutation invariant. Moreover, if $\kernelf^S(\node,\node',\overline{\node},\overline{\node}')$ and $\kernelf^A(\node,\node',\overline{\node},\overline{\node}')$ are the symmetric and anti-symmetric forms of a kernel $\kernelf(\node,\node',\overline{\node},\overline{\node}')$ obtained with the projections given in Definition~\ref{symmasymmkdef}, then the permutation invariant form of the kernel obtained with the projection given in Definition~\ref{strangekdef} can be expressed as the sum of the symmetric and and anti-symmetric forms:
%\[
%\kernelf^{\pi\pointspace^2}(\node,\node',\overline{\node},\overline{\node}')=\kernelf(\node',\node,\overline{\node}',\overline{\node}),
%\]
%that is, the latter kernel is obtained from the former by a permutation of arguments.:
\begin{align*}
\kernelf^{PI}(\node,\node',\overline{\node},\overline{\node}')
=&\kernelf^S(\node,\node',\overline{\node},\overline{\node}')+\kernelf^A(\node,\node',\overline{\node},\overline{\node}')\;.
\end{align*}\hfill\ensuremath{\blacksquare}
\end{lemma}
%\begin{proof}
%Direct computation.
%\end{proof}

%The RKHS of the symmetric and anti-symmetric kernels can be expressed as the 
%In the forthcoming considerations, we express the feature mappings corresponding to the symmetric and anti-symmetric kernels as follows:
%next note that the representers of evaluation for the RKHSs corresponding to the symmetric and anti-symmetric kernels can be, respectively, expressed as:
%With every permutation $P$ of $\{1, 2, \ldots, N\}$ we associate a permutation operator $\hat{P}$ on $\mathcal{H}$ that maps a function $\psi$ to a function $\hat{P}\psi$ defined by
%\[
%( \hat{P}\psi)(q_1 , q_2 , \ldots, q_N ) = \psi(q_{P(1)} , q_{P(2)}, \ldots, q_{P(N)})
%\]

%$\symm$ and $\asymm$ are projectors onto mutually orthogonal subspaces $\mathcal{H}^S$ and $\mathcal{H}^A$ of $\mathcal{H}$. Each subspace is itself a Hilbert space; $\mathcal{H}^S$ is the space of symmetrical functions, and $\mathcal{H}^A$ is the space of antisymmetrical functions.

\subsection{Spectral Analysis of Pairwise Kernels}

We next study the relationship between the integral operators of the permutation invariant, symmetric and anti-symmetric kernels to the corresponding integral operator of the original kernel they were constructed from.
\begin{theorem}\label{operatortheorem}
Let $\kernelf(\node,\node',\overline{\node},\overline{\node}')$ be an arbitrary pairwise kernel and let $\kernelf^{PI}$, $\kernelf^S$ and $\kernelf^A$ be its permutation invariant, symmetric and anti-symmetric forms. Moreover, let $\koper_\kernelf$, $\koper_{\kernelf^{PI}}$, $\koper_{\kernelf^{S}}$ and $\koper_{\kernelf^{A}}$ be the integral operators of the kernels $\kernelf$, $\kernelf^{PI}$, $\kernelf^S$ and $\kernelf^A$, respectively. Then,
\begin{align}
\koper_{\kernelf^P}&={\shufflem^\mu}^*\koper_{\kernelf}\shufflem^\mu\nonumber\\
\koper_{\kernelf^S}&={\symm^\mu}^*\koper_{\kernelf}\symm^\mu\nonumber\\
\koper_{\kernelf^A}&={\asymm^\mu}^*\koper_{\kernelf}\asymm^\mu\nonumber\\
\koper_{\kernelf^{PI}}&=\frac{1}{2}\left(\koper_{\kernelf}+{\shufflem^\mu}^*\koper_{\kernelf}\shufflem^\mu\right)\label{firstpiopereq}\\
&={\symm^\mu}^*\koper_{\kernelf}\symm^\mu +{\asymm^\mu}^*\koper_{\kernelf}\asymm^\mu\;,\label{secondpiopereq}
\end{align}
where
\begin{align*}
\shufflem^\mu:&L^2(\pointspace^2, \mu) \rightarrow L^2(\pointspace^2, \mu)\\
&\hypothesis(\overline{\node},\overline{\node}')\mapsto\frac{\mu(\overline{\node}',\overline{\node})}{\mu(\overline{\node},\overline{\node}')}\hypothesis(\overline{\node}',\overline{\node})
\end{align*}
is an operator to which we refer as the permutation operator with respect to the measure $\mu$, and whose adjoint is
\begin{align*}
\hypothesis(\overline{\node},\overline{\node}')\mapsto\frac{\mu(\overline{\node},\overline{\node}')}{\mu(\overline{\node}',\overline{\node})}\hypothesis(\overline{\node}',\overline{\node})\;,
\end{align*}
and
\begin{align*}
\symm^\mu=&\frac{1}{2}\left(\bm{I}+\shufflem^\mu\right)\\
\asymm^\mu=&\frac{1}{2}\left(\bm{I}-\shufflem^\mu\right)
\end{align*}
are projection operators to which we refer as the symmetrizer and anti-symmetrizer with respect to the measure $\mu$, and $\bm{I}$ is the identity operator of $L^2(\pointspace^2, \mu)$.
% \begin{align*}
% \asymm^\mu:&L^2(\pointspace^2, \mu) \rightarrow L^2(\pointspace^2, \mu)\\
% &\hypothesis(\overline{\node},\overline{\node}')\mapsto\frac{1}{2}\left(\hypothesis(\overline{\node},\overline{\node}')-\frac{\mu(\overline{\node}',\overline{\node})}{\mu(\overline{\node},\overline{\node}')}\hypothesis(\overline{\node}',\overline{\node})\right)
% \end{align*}
% %is an idempotent self adjoint operator. We observe that:
% is a projection operator to which we refer as the anti-symmetrizer with respect to the measure $\mu$, and whose adjoint operator is
% \begin{align*}
% {\asymm^\mu}^*:&L^2(\pointspace^2, \mu) \rightarrow L^2(\pointspace^2, \mu)\\
% &\hypothesis(\overline{\node},\overline{\node}')\mapsto\frac{1}{2}\left(\hypothesis(\overline{\node},\overline{\node}')-\frac{\mu(\overline{\node},\overline{\node}')}{\mu(\overline{\node}',\overline{\node})}\hypothesis(\overline{\node}',\overline{\node})\right)\;.
% \end{align*}
%The integral operators of the pairwise kernels $\{\koper_{\kernelf^{S}},\koper_{\kernelf^{A}},\koper_{\kernelf^{PI}}\}$
\end{theorem}
\noindent See Section~\ref{operatorproof} for a proof.

%We obtained the expression $\koper_{\kernelf^A}={\asymm^\mu}^*\koper_{\kernelf}\asymm^\mu$ for the integral operator of the anti-symmetric kernel $\kernelf^A$.
Next, we look on what can be said about the spectrum of the integral operators considered in the above theorem. This consideration can be divided into the important special case of the measure $\mu$ being symmetric, that is
\[
\mu(\overline{\node},\overline{\node}')=\mu(\overline{\node}',\overline{\node}),\forall(\overline{\node},\overline{\node}')\in\pointspace^2
\]
and to the general case. The measure is symmetric, for example, in various types of ranking and preference learning tasks as is considered more in detail below. In addition, many other pairwise learning problems with non-symmetric measure can be turned to problems with a symmetric measure by the technique known as virtual examples. That is, whenever a datum $(\node,\node')$ is drawn from $\mu$, one also introduces a virtual example $(\node',\node)$ with the same output if the problem is considered to be symmetric or with the opposite output in the anti-symmetric case. With symmetric $\mu$, the symmetrizer and anti-symmetrizer projections do not depend on the measure and we denote them simply as $\symm$ and $\asymm$.

\begin{corollary}
%We have the following trace equalities in the general case of $\mu$ not necessarily being symmetric:
%\begin{align}
%\operatorname{trace}(\koper_{\kernelf^{PI}})&=\frac{1}{2}\operatorname{trace}(\koper_{\kernelf}+\koper_{\kernelf^P})\label{generaltraceequalityone}\\
%&=\operatorname{trace}(\koper_{\kernelf^S}+\koper_{\kernelf^A})\label{generaltraceequalitytwo}\;.
%\end{align}
If $\lambda_{i}^{\kernelf}$, $\lambda_{i}^{\kernelf^S}$ and $\lambda_{i}^{\kernelf^A}$ denote the eigenvalues of $\koper_{\kernelf}$, $\koper_{\kernelf^S}$ and $\koper_{\kernelf^A}$, respectively, then
\begin{align}\label{eigenineq}
\lambda_{i}^{\kernelf^S}\leq\lambda_{i}^{\kernelf}\textnormal{ and }\lambda_{i}^{\kernelf^A}\leq\lambda_{i}^{\kernelf}\textnormal{ for }i=1,2,\ldots
\end{align}
%we also have the following additional trace equalities:
%\begin{align}
%\operatorname{trace}(\koper_{\kernelf^{PI}})=\operatorname{trace}(\koper_{\kernelf})=\operatorname{trace}(\koper_{\kernelf^P})\;.
%\end{align}
If $\mu$ is symmetric, the set of operators $\{\symm,\asymm,\koper_{\kernelf^{S}},\koper_{\kernelf^{A}},\koper_{\kernelf^{PI}}\}$ commutes, which in turn indicates that they can be diagonalized simultaneously as follows:
%The operators $\koper_{\kernelf^{S}}$, $\koper_{\kernelf^{A}}$ and  $\koper_{\kernelf^{PI}}$ can be eigen decomposed as
\begin{eqnarray*}
\koper^{PI}&=&\evecmatrix\evalmatrix^{PI}\evecmatrix^*\;,\\
\koper^S&=&\evecmatrix\evalmatrix^S\evecmatrix^*\;,\\
\koper^A&=&\evecmatrix\evalmatrix^A\evecmatrix^*\;,\\
\symm&=&\evecmatrix\idmatrix^S\evecmatrix^*\;,\\
\asymm&=&\evecmatrix\idmatrix^A\evecmatrix^*\;,\\
\end{eqnarray*}
where $\evecmatrix$ is an unitary operator containing the eigenfunctions and $\evalmatrix^{PI}$, $\evalmatrix^S$, $\evalmatrix^A$, $\idmatrix^S$ and $\idmatrix^A$ are operators containing the corresponding eigenvalues of the five operators under consideration, and
\begin{equation}\label{evaleq}
\evalmatrix^{PI}=\evalmatrix^S+\evalmatrix^A\;,
\end{equation}
if the eigenvalues are arranged in the order determined by the order of eigenfunction in $\evecmatrix$.

%Given that the kernel $\kernelf^{PI}$ is universal, the corresponding kernel matrix has a full rank unless the training set contains several occurrences of the same data point. Assuming the full rank of $\kernelm^{PI}$, we observe that
%\begin{eqnarray*}
%\textrm{Null}(\kernelm^{PI})&=&\{\bm{0}\},\\
%\textrm{Null}(\kernelm^S)&=&\textrm{Null}(\symm)\textnormal{, and}\\
%\textrm{Null}(\kernelm^A)&=&\textrm{Null}(\asymm).
%\end{eqnarray*}
%This, together with the fact that the three kernel matrices share the same eigensystem and with the equality (\ref{evaleq}), indicates that $\kernelm^\asymm$ has the same eigenvalues as $\kernelm^{PI}$ except that $(\dimone^2-\dimone)/2+\dimone$ of them are replaced with zeros, that is, those corresponding to the symmetric $\dimone\times\dimone$-matrices in the null space of $\asymm$.

Finally, if $\mu$ is symmetric, then
\begin{align}\label{majorizationownage}
\koper_{\kernelf^{PI}}\prec\koper_{\kernelf}
\end{align}
(e.g. the sequence of eigenvalues of $\koper_{\kernelf}$ majorizes the sequence of eigenvalues of $\koper_{\kernelf^{PI}}$).
\end{corollary}
\begin{proof}
%The trace equalities (\ref{generaltraceequalityone}) and (\ref{generaltraceequalitytwo}) follow directly from the properties of trace applied on the operator equalities (\ref{firstpiopereq}) and (\ref{secondpiopereq}).

Since $\asymm^\mu$ is a projection matrix, $\asymm^\mu(L^2(\pointspace^2, \mu))\subset L^2(\pointspace^2, \mu)$, this constrains the action of the integral operator $\koper_{\kernelf}$ onto the range of $\asymm^\mu$, which is a subspace of $L^2(\pointspace^2, \mu)$. The eigenfunctions $\phi_{i}$ associated with nonzero eigenvalues $\lambda_i$ of $\koper_{\kernelf^A}$ belong to this subspace, and satisfy (\citet{Aronszajn1948RayleighRitz}):
\[
\koper_{\kernelf}\phi_i - \lambda_i \phi_i=p \text{ with } p \perp \asymm^\mu(L^2(\pointspace^2, \mu))\;.
\]
Since $\asymm^\mu(L^2(\pointspace^2, \mu))\subset L^2(\pointspace^2, \mu)$, we can use a well known theorem (see e.g. \citet{Aronszajn1948RayleighRitz} and references therein) to obtain:
\[
\textnormal{ and }\lambda_{i}^{\kernelf^A}\leq\lambda_{i}^{\kernelf}\textnormal{ for }i=1,2,\ldots
\]
and the case with $\koper_{\kernelf^S}$ goes analogously.

% With symmetric $\mu$, the permutation operator $\shufflem$ is orthogonal, and hence with the cyclic property of the trace, we get 
% \begin{align*}
% \operatorname{trace}(\koper_{\kernelf^{P}})=&\operatorname{trace}(\shufflem\koper_{\kernelf}\shufflem))\\
% =&\operatorname{trace}(\koper_{\kernelf}\shufflem\shufflem)\\
% =&\operatorname{trace}(\koper_{\kernelf})\\
% =&\frac{1}{2}(\operatorname{trace}(\koper_{\kernelf})+\operatorname{trace}(\koper_{\kernelf^P}))\\
% =&\operatorname{trace}(\koper_{\kernelf^{PI}})\;.
% \end{align*}
We observe that, with symmetric $\mu$, the operators $\symm$ and $\asymm$ are self-adjoint, and hence orthogonal projections. Furthermore, they are orthogonal with each other, that is
\begin{align}
\symm\asymm=\asymm\symm=0\;,
\end{align}
and hence the set $\{\symm,\asymm,\koper_{\kernelf^{S}},\koper_{\kernelf^{A}},\koper_{\kernelf^{PI}}\}$ of operators commutes, and therefore, according to Corollary~\ref{commutationcoro}, they share the same eigenfunctions.

Finally, (\ref{majorizationownage}) follows from the Uhlmann's theorem, since we can define an operation:
\begin{align*}
\Gamma:&\mathcal{T}(L^2(\pointspace^2, \mu))\rightarrow\mathcal{T}(L^2(\pointspace^2, \mu))\\
&\koper_{\kernelf}\mapsto\frac{1}{2}\left(\koper_{\kernelf}+\shufflem\koper_{\kernelf}\shufflem\right)\label{firstpiopereq}
\end{align*}
for which $\koper_{\kernelf^{PI}}=\Gamma(\koper_{\kernelf})$ and which is doubly stochastic, because it is both trace preserving (as shown above), unital due to $\shufflem\shufflem=\idmatrix$, and the set of Kraus operators fulfilling (\ref{krausform}) is $\{\frac{1}{2}\idmatrix,\frac{1}{2}\shufflem\}$.
\end{proof}
It is interesting to note the following observation about the common eigensystem of the operators $\{\symm,\asymm,\koper_{\kernelf^{S}},\koper_{\kernelf^{A}},\koper_{\kernelf^{PI}}\}$:
\begin{remark}
All the eigenfunctions of $\koper_{\kernelf^{PI}}$ are either symmetric or anti-symmetric, and the corresponding eigenvalues are cleared to zeros when one applies $\symm$ or $\asymm$. Since $\symm$ and $\asymm$ are orthogonal projections, their eigenvalues are either zeros or ones, and the ones in $\symm$ correspond to the symmetric functions and zeros to the anti-symmetric ones, and vice versa for $\asymm$.
\end{remark}

\section{Error Bounds}

Let
\begin{equation}\label{expectedrisk}
\defect(\predfun)=\int_{\mathcal{X}\times\mathcal{Y}}\lossfunction(\predfun(x),y)d\rho(x,y),
\end{equation}
where $\lossfunction$ is a loss function, denote the expected risk of $\predfun$. For the squared loss, the minimizer of (\ref{expectedrisk}) is the so-called regression function
\[
\predfun^*(x)=\int_{\mathcal{Y}}yd\rho(x,y).
\]
The hypothesis spaces under our consideration in this paper do not necessarily include the regression function, and hence another quantity of interest is the error associated to the given RKHS $\hypspace$:
\[
\inf_{\predfun\in\hypspace}\defect(\predfun).
\]
If we have a prior knowledge, for example, that the underlying regression function is anti-symmetric, then we can immediately assume that the errors associated to a kernel $\kernelf$ and its anti-symmetric counterpart $\kernelf^A$ are equal. That is, we do not lose any expressiveness by restricting our hypothesis space to anti-symmetric functions. The next question is whether we can gain anything with the restriction.

Our next quantity of interest is the minimizer $\predfun_{\tset,\regparam}$ of the regularized empirical risk on a training set $\tset$ and a regularization parameter $\regparam$. In particular, we aim to analyze the effect of using either the permutation-invariant, symmetric, or anti-symmetric forms instead of the original kernel on the discrepancy
\[
\defect(\predfun_{\regparam,\tset})-\textnormal{inf}_{\predfun\in\hypspace(\kernelf)}\defect(\predfun)
\]
known in the literature as the excess error. %In the literature, there are many results
%  (see e.g. \citet{Vito2005inverse} and references therein) indicating that the expected prediction error of regularized kernel methods obey the following type of probabilistic upper bounds. Namely, for any $0<\eta<1$, it holds that
% \begin{equation}\label{boundeq}
% %\min_{\dualparvect}J(\dualparvect)+\mathcal{C}\left(\textnormal{trace}\left(\kernelm(\kernelm+\regparam\idmatrix)^{-1}\right)\right),
% %\min_{\dualparvect}J(\dualparvect)+\mathcal{C}\left(\kappa\right),
% P\left[\defect(\predfun_{\regparam,\tset})-\textnormal{inf}_{\predfun\in\hypspace(\kernelf)}\defect(\predfun) \leq \errbound(\regparam,\kernelf,\eta)\right] \geq 1-\eta.
% \end{equation}
% where $P[\cdot]$ denotes the probability and $\errbound(\regparam,\kernelf,\eta)$ is a complexity term depending on the kernel, the amount of regularization, and the confidence level $\eta$.

Following \citet{HsuK014randomdesign}, we split the consideration of the excess error into three parts:
\begin{align*}
\defect(\predfun_{\regparam,\tset})-\textnormal{inf}_{\predfun\in\hypspace(\kernelf)}\defect(\predfun)\leq \epsilon_{rg}+\epsilon_{bs}+\epsilon_{vr}+2(\sqrt{\epsilon_{rg}\epsilon_{bs}}+\sqrt{\epsilon_{rg}\epsilon_{vr}}+\sqrt{\epsilon_{bs}\epsilon_{vr}})\;,
\end{align*}
where $\epsilon_{rg}$, $\epsilon_{bs}$, and $\epsilon_{vr}$
%\begin{align*}
%\epsilon_{rg}=,\phantom{W}\epsilon_{bs}=,\phantom{W}\epsilon_{vr}=
%\end{align*}
are, respectively, the bias caused by regularization, the bias caused by the random drawing of the training inputs, and the variance caused by noise in the outputs. We briefly consider each of these in turn in the following subsections.

\subsection{Effective Dimension}

As discussed by \citet{HsuK014randomdesign} and also earlier by many other authors (see e.g. (\citet{Zhang2005effective,Caponnetto2007optimalrates}), the variance term $\epsilon_{vr}$ can be roughly characterized with a concept known as the effective dimension:
%A particular class of complexity terms that has appeared in several works depend on 
\begin{definition}[Effective dimension]
The effective dimension $\effdim(\kernelf,\mu,\regparam)$ of the kernel $\kernelf$ with respect to the measure $\mu$ and the regularization parameter value $\regparam>0$ is defined as:
\begin{align*}%\label{effdim}
%\textnormal{trace}\left((\koper+\regparam\idmatrix)^{-1}\koper\right)
\effdim(\kernelf,\mu,\regparam)=\sum_{i=1}^\infty\frac{\lambda_i}{\lambda_i+\lambda}\;,
\end{align*}
where $\lambda_i$ are the eigenvalues of the integral operator of the kernel $\kernelf$.
\end{definition}

%The following lemma recently proposed by \citet{vallee2007infschurconvexity} generalizes the concept of Schur-convexity for Hilbert space operators:
%\begin{lemma}
%Let $(\sigma_i)_{i=1}^\infty$ be a sequence of real numbers arranged into a nondecreasing order. The functional
%\begin{align*}
%\Xi:O(\mathcal{T})&\rightarrow\mathbb{R}\\
%\Xi(\koper)&\mapsto\sum_{i=1}^\infty\sigma_i\lambda_i\;,
%\end{align*}
%where $\lambda_i$ are the eigenvalues of $\koper$ arranged in a decreasing order, is convex and lower semicontinuous.
%\end{lemma}

%\begin{lemma}\label{concavelemma}
%\end{lemma}
The next result shows that the eigenvalue majorization of the integral operators of kernels is connected to the effective dimension of the kernels:
\begin{proposition}\label{efdimmajorizationcoro}
Let $\kernelf_1$ and $\kernelf_2$ be kernels, and $\koper_1$ and $\koper_2$ their integral operators with measure $\mu$, with $\operatorname{trace}(\koper_1)=\operatorname{trace}(\koper_2)$. Then,
\[
\koper_2\prec\koper_1 \Rightarrow\effdim(\kernelf_2,\mu,\regparam)>\effdim(\kernelf_1,\mu,\regparam)\phantom{W}\forall\regparam>0\;.
\]
\end{proposition}
\begin{proof}
We recollect the following result recently proven by \citet{Mari2014Quantumstatemajorization} that extends a well-known result for sequences of infinite lengths. Let $\bm{r}=(r_i)_{i=1}^\infty\in c_0^*$ and $\bm{s}=(s_i)_{i=1}^\infty\in c_0^*$ with $\sum_{i=1}^\infty r_i=\sum_{i=1}^\infty s_i=1$. Then,
\begin{align*}
\bm{r}\prec\bm{s}\Leftrightarrow\sum_{i=1}^\infty\rho(r_i)\geq\sum_{i=1}^\infty\rho(s_i)\;.
\end{align*}
for all real non-negative strictly concave function $\rho$ defined on the segment $[0, 1]$. The result follows immediately (with scaling the eigenvalues), since $\rho(r)=r/(r+\regparam)$ is real-valued, non-negative and strictly concave for $r,\regparam>0$.
\end{proof}

Given the above analysis of the eigensystems of the considered pairwise kernels, we end up to the following results about their effective dimensions: 
\begin{theorem}\label{effdimtheorem}
If $\kernelf$ is a pairwise kernel, then 
\begin{align}\label{symmefdim}
\effdim(\kernelf^S,\mu,\regparam)\leq\effdim(\kernelf,\mu,\regparam)
\end{align}
and
\begin{align}\label{asymmefdim}
\effdim(\kernelf^A,\mu,\regparam)\leq\effdim(\kernelf,\mu,\regparam)\;.
\end{align}
If the measure $\mu$ is symmetric, we also have
\begin{align}\label{piefdim}
\effdim(\kernelf,\mu,\regparam)\leq\effdim(\kernelf^{PI},\mu,\regparam)\;.
\end{align}
\end{theorem}
\begin{proof}
The inequalities (\ref{symmefdim}) and (\ref{asymmefdim}) follow straightforwardly from (\ref{eigenineq}), and the inequality (\ref{piefdim}) follows from Corollary~\ref{efdimmajorizationcoro} and (\ref{majorizationownage}).
\end{proof}

\subsection{Approximation Analysis}
We next rurn our attention to the bias caused by the random drawing of the training inputs. According to \citet{HsuK014randomdesign}, this bias is affected, in addition to the above considered effective dimension and the regularization bias considered below, by the approximation error caused by the hypothesis space being too limited.
%, which indicates that our hypothesis space does not necessarily contain the regression function. 
In contrast, the approximation error is zero if the hypothesis space contains the regression function or functions that can approximate it arbitrarily closely. To guarantee that the hypothesis space is expressive enough to approximate any function, we may use kernels that are universal. On the other hand, if we have prior knowledge about the properties of the regression function, for example, if we know it to be symmetric or anti-symmetric, we may restrict the hypothesis space accordingly.

Related to the bias by random design, we also point out a recent result by \citet{brunner2012pairwise} which shows an equivalence between the use of a symmetric pairwise kernel and the original kernel with a symmetrized training set. We omit its detailed consideration here due to lack of space. 

To formalize these concepts, we first recollect the definition of universal kernels.
\begin{definition}[\citet{Steinwart2002consistency}]\label{kerneluniversalitydef}
A continuous kernel $\kernelf$ on a compact metric space $\anyspace$ (i.e. $\anyspace$ is closed and bounded) is called universal if the RKHS induced by $\kernelf$ is dense in $C(\anyspace)$, where $C(\anyspace)$ is the space of all continuous functions $\predfun : \anyspace \rightarrow \mathbb{R}$.
% That is, for every function $\predfun\in C(\anyspace)$ and every $\epsilon > 0$, there exists a set of input points $\{x_i \}_{i=1}^\tsize \in \anyspace$ and real numbers $\{\alpha_i\}_{i=1}^\tsize$, with $\tsize\in \mathbb{N}$, such that
%\begin{equation*}%\label{uniformconv}
%\max_{x\in \anyspace}\left\{\left\arrowvert \predfun(x)-\sum_{i=1}^\tsize\alpha_i\kernelf(x_i,x)\right\arrowvert\right\}\leq\epsilon.
%\end{equation*}
%Accordingly, the hypothesis space induced by the kernel $\kernelf$ can approximate any function in $C(\anyspace)$ arbitrarily well, and hence it has the universal approximating property.
\end{definition}
Accordingly, the hypothesis space induced by the kernel $\kernelf$ can approximate any function in $C(\anyspace)$ arbitrarily well, and hence it is said to have the universal approximating property.

While the universal approximating property guarantees that the RKHS can, in theory, learn any concept, we do not necessarily have a need for it if we have prior knowledge about certain properties of the concept to be learned.
% If we know that the concept belongs to a certain subset of continuous functions, say $\funset\subseteq C(\anyspace)$, it is enough that the RKHS can approximate the functions in $\funset$ only. For example, knowing in advance that the concept is linear, it is enough that the hypothesis space contains all functions linear in the inputs.
Thus, we also define an analogous concept for non-universal kernels: 
\begin{definition}\label{approxdef}
Let $\kernelf$ be a continuous kernel $\kernelf$ on a compact metric space $\anyspace$ and let $\funset\subseteq C(\anyspace)$.
If $\funset\subseteq \hypspace(\kernelf)$, the definition of RKHS indicates that, for every function $\predfun\in C(\anyspace)$ and every $\epsilon > 0$, there exists a set of input points $\{x_i \}_{i=1}^\tsize \in \anyspace$ and real numbers $\{\alpha_i\}_{i=1}^\tsize$, with $\tsize\in \mathbb{N}$, such that
\begin{equation*}%\label{uniformconv}
\max_{x\in \anyspace}\left\{\left\arrowvert \predfun(x)-\sum_{i=1}^\tsize\alpha_i\kernelf(x_i,x)\right\arrowvert\right\}\leq\epsilon.
\end{equation*}
Accordingly, the hypothesis space induced by the kernel $\kernelf$ can approximate any function in $\funset$ arbitrarily well, and hence we say that the RKHS $\hypspace(\kernelf)$ can approximate $\funset$.
\end{definition}
Armed with the above definitions, we present the next result characterizing the approximation properties of the symmetric and anti-symmetric kernels:
\begin{theorem}\label{generalantisymmetrictheorem}
Let $\funset\subseteq C(\pointspace^2)$ be an arbitrary set of continuous functions, and let
\begin{align*}
% \mathcal{A}=\left\{\asymfun\mid \asymfun\in \funset,\asymfun(\node, \node')=-\asymfun(\node',\node)\right\}
\mathcal{S}=\left\{\asymfun\mid\arbfun\in \funset,\asymfun(\node, \node')=\arbfun(\node, \node')+\arbfun(\node',\node)\right\}\\
\mathcal{A}=\left\{\asymfun\mid\arbfun\in \funset,\asymfun(\node, \node')=\arbfun(\node, \node')-\arbfun(\node',\node)\right\}
\end{align*}
be the sets of symmetric and anti-symmetric functions determined by $\funset$. Moreover, let $\kernelf(\node,\node',\overline{\node},\overline{\node}')$ be a kernel on $\pointspace^2$ and let $\kernelf^S(\node,\node',\overline{\node},\overline{\node}')$ and $\kernelf^A(\node,\node',\overline{\node},\overline{\node}')$ be the corresponding symmetric and anti-symmetric kernels. If $\funset\subseteq\hypspace\left(\kernelf\right)$, then $\mathcal{S}\subseteq\hypspace\left(\kernelf^S\right)$ and $\mathcal{A}\subseteq\hypspace\left(\kernelf^A\right)$.
\end{theorem}
\noindent See Section~\ref{generalantisymmetricproof} for a proof.
%An analogous result can be shown for the symmetric kernel. 
This theorem is a generalization of the result of \citet{waegeman2012learninggraded}, who proved that this result holds for the special cases of the symmetric and anti-symmetric Kronecker product kernel.

As an example of an anti-symmetric kernel popularly used in the machine learning literature, we may consider the following one originally analyzed by \citep{Herbrich2000}. Given a base kernel $\kernelf^{\pointspace}(\node,\overline{\node})$ over the objects, the pairwise learning to rank approach corresponds to using the following transitive pairwise kernel: 
\[
\frac{1}{4}\left(
\kernelf^{\pointspace}(\node,\overline{\node})
-\kernelf^{\pointspace}(\node',\overline{\node})
-\kernelf^{\pointspace}(\node,\overline{\node}')
+\kernelf^{\pointspace}(\node',\overline{\node}')
\right)
\]
In the theoretical framework considered in this paper, this kernel can be interpreted as the anti-symmetrization of the pointwise kernel 
$\kernelf(\node,\node',\overline{\node},\overline{\node}')=\kernelf^{\pointspace}(\node,\overline{\node})$, that simply ignores the second pair. The approximation properties of this kernel are thus formalized in the following corollary:
\begin{corollary}
Let
\[
\mathcal{R}=\left\{t\mid t\in C(\pointspace^2),\exists r\in C(\pointspace),t(\node, \node')=r(\node)-r(\node')\right\}
%\mathcal{T}=\left\{t\mid t\in C(\pointspace^2),t(\node, \node')=-t(\node',\node)\right\}
\]
be the set of all continuous ranking functions from $\pointspace^2$ to $\mathbb{R}$. If $\kernelf^{\pointspace}(\node,\overline{\node})$ on $\pointspace$ is universal, then the RKHS of the transitive kernel \citep{Herbrich2000} defined as
$\kernelf_T(\node,\node',\overline{\node},\overline{\node}')=$
\begin{equation}\label{transitivekernel}
\frac{1}{4}\left(
\kernelf^{\pointspace}(\node,\overline{\node})
-\kernelf^{\pointspace}(\node',\overline{\node})
-\kernelf^{\pointspace}(\node,\overline{\node}')
+\kernelf^{\pointspace}(\node',\overline{\node}')
\right)
\end{equation}
can approximate $\mathcal{R}$. 
%, that is, for every function $t\in T(\pointspace^2)$ and every $\epsilon > 0$, there exists a function $h$ in the RKHS induced by the kernel (\ref{askernel}), such that
%\begin{equation*}
%\max_{(\node,\node')\in \pointspace^2}\left\{\left\arrowvert t(\node,\node')-\hypothesis(\node,\node')\right\arrowvert\right\}\leq\epsilon \,.
%\end{equation*}
\end{corollary}
\begin{proof}
We select
\[
\funset=\left\{\predfun\mid \predfun\in C(\pointspace^2),\exists r\in C(\pointspace),t(\node, \node')=r(\node)\right\}
\]
and apply Theorem~\ref{generalantisymmetrictheorem}.
\end{proof}

\subsection{Regularization Bias}

The following expression of the bias caused by regularization is known in the literature (see e.g. \citet{HsuK014randomdesign}) but we show it here for the completeness, because we express it in somewhat different form.
\begin{lemma}\label{expressionlemma}
Let $\predfun$ be the regression function and $\koper_\kernelf$ the integral operator of a kernel $\kernelf$. Further, let $\hypothesis^\regparam$ be the minimizer of the regularized mean squared error
\begin{align}\label{populationobjfun}
\int_\anyspace\left(\predfun-\lmutorkhs_\kernelf^*\hypothesis\right)^2d\mu+\regparam\Arrowvert\hypothesis\Arrowvert_{\hypspace(\kernelf)}\;,
\end{align}
and let $\predfun^\regparam=\lmutorkhs_\kernelf^*\hypothesis^\regparam$. Then, $\predfun^\regparam$ can be expressed as
% and a level of regularization $\regparam>0$, the bias caused by regularization 
\begin{align*}
\predfun^\regparam&=\bm{V}\bm{\Lambda}(\bm{\Lambda}+\regparam\idmatrix)^{-1}\bm{V}^*\predfun\;,
\end{align*}
and the bias caused by regularization as
\[
\epsilon_{rg}(\predfun,\koper_\kernelf,\regparam)=\regparam^{2}\left\langle\predfun,\left(\koper_\kernelf+\regparam\idmatrix\right)^{-2}\predfun\right\rangle\;,
\]
where $\koper_\kernelf=\bm{V}\bm{\Lambda}\bm{V}^*$ is the eigen decomposition of $\koper_\kernelf$, and the operator-vector products are in $L^2(\anyspace, \mu)$.
\end{lemma}
\noindent See Section~\ref{expressionproof} for a proof.

%On an operator Kantorovich inequality for positive linear maps
%Some operator inequalities for positive linear maps

%\begin{align}
%\frac{(\alpha+\beta)^2}{4\alpha\beta}\Psi(\anymatrix^{-2}) \geq \frac{(\alpha+\beta)^2}{4\alpha\beta}\Psi(\anymatrix^{-1})^{2}\geq\Psi(\anymatrix^{-2}) \geq \Psi(\anymatrix^{2})^{-1} \geq \frac{(\alpha+\beta)^2}{4\alpha\beta}\Psi(\anymatrix^{-2})
%\end{align}

%\begin{proof}
% \begin{align*}
% \epsilon_{rg}&=\regparam^{2}\predfun^*\bm{V}\left(\bm{\Lambda}+\regparam\idmatrix\right)^{-2}\bm{V}^*\predfun\\
% &=\regparam^{2}\predfun^*\symm\bm{V}\left(\bm{\Lambda}+\regparam\idmatrix\right)^{-2}\bm{V}^*\symm\predfun\\
% &=\regparam^{2}\predfun^*\bm{V}\idmatrix_S\left(\bm{\Lambda}+\regparam\idmatrix\right)^{-2}\idmatrix_S\bm{V}^*\predfun\\
% &=\regparam^{2}\predfun^*\bm{V}\idmatrix_S\left(\left(\bm{\Lambda}_{\kernelf^S}+\regparam\idmatrix\right)^{-2}+\left(\bm{\Lambda}_{\kernelf^A}+\regparam\idmatrix\right)^{-2}\right)\idmatrix_S\bm{V}^*\predfun\\
% &=\regparam^{2}\predfun^*\bm{V}\left(\bm{\Lambda}_{\kernelf^S}+\regparam\idmatrix\right)^{-2}\bm{V}^*\predfun
% \end{align*}
%\end{proof}

%With a similar argument, one can show the following result for the bias caused by the random selection of the training inputs.

Interestingly, if the same value of the regularization parameter is used for both the original kernel and its permutation invariant, symmetric or anti-symmetric forms, the type depending on the prior knowledge we have about the regression function, the regularization bias may get worse even if we use the correct type of modification of the kernel. In fact, one can find examples of symmetric regression functions for which the kernel symmetrization decreases the bias and other symmetric regression functions for which the bias is increased. However, the increase or decrease of the bias is rather mild and it is characterized by the following result:
\begin{theorem}\label{regerrpropo}
Let us assume $\kernelf$ $\max_{(\node,\node')\in\pointspace^2}\kernelf(\node,\node',\node,\node')=1$. This assumption can be done without losing generality due to the kernels being bounded.

If the measure $\mu$ is symmetric and the regression function is symmetric (anti-symmetric), the bias caused by regularization is the same for the kernels $\kernelf^{PI}$ and $\kernelf^{S}$ ($\kernelf^{PI}$ and $\kernelf^{A}$) with all values of $\regparam$. Moreover, the bias caused by regularization with the amount $\regparam$ for the kernel $\kernelf$ and $\kernelf^{PI}$ has the following relationship:
%is the same for the kernels $\kernelf^{PI}$ and $\kernelf^{S}$ ($\kernelf^{PI}$ and $\kernelf^{A}$), if the
\begin{align*}
% \frac{4\alpha^2\beta^2}{(\alpha^2+\beta^2)^2}\epsilon_{rg}(\predfun,\koper_{\kernelf},\regparam)
% \leq\epsilon_{rg}(\predfun,\koper_{\kernelf^{PI}},\regparam)
% \leq\frac{(\alpha+\beta)^2}{4\alpha\beta}\epsilon_{rg}(\predfun,\koper_{\kernelf},\regparam)\;,
\left(1-\frac{(\regparam^2-(\regparam+1)^2)^2}{(\regparam^2+(\regparam+1)^2)^2}\right)\epsilon_{rg}(\predfun,\koper_{\kernelf},\regparam)
&\leq\epsilon_{rg}(\predfun,\koper_{\kernelf^{PI}},\regparam)\\
&\leq\left(1+\frac{1}{4\regparam^2+4\regparam}\right)\epsilon_{rg}(\predfun,\koper_{\kernelf},\regparam)\;.
\end{align*}
\end{theorem}
\noindent See Section~\ref{regerrproof} for a proof.

\section{Proofs}

\subsection{Proof of Theorem~\ref{operatortheorem}}\label{operatorproof}

\begin{proof}
We begin by considering the integral operator of the anti-symmetric kernel. For $\hypothesis,g\in L^2(\pointspace^2, \mu)$,
\begin{align*}
\langle\koper_{\kernelf^A}\hypothesis, g\rangle_{L^2(\pointspace^2, \mu)}
%&=\left\langle\int_{\pointspace^2} \kernelf^A_{(\overline{\node},\overline{\node}')}\hypothesis(\overline{\node},\overline{\node}')d\mu,\int_{\pointspace^2} \kernelf^A_{(\overline{\node},\overline{\node}')}g(\overline{\node},\overline{\node}')d\mu\right\rangle\\
&=\int_{\pointspace^2}g(\node,\node)\left(\int_{\pointspace^2} \kernelf^A(\node,\node',\overline{\node},\overline{\node}')\hypothesis(\overline{\node},\overline{\node}')d\mu\right) d\mu\\
&=\int_{\pointspace^2}\int_{\pointspace^2} g(\node,\node)\kernelf^A(\node,\node',\overline{\node},\overline{\node}')\hypothesis(\overline{\node},\overline{\node}')d\mu d\mu\\
&=\frac{1}{4}\int_{\pointspace^2}\int_{\pointspace^2} g(\node,\node')\kernelf(\node,\node',\overline{\node},\overline{\node}')\hypothesis(\overline{\node},\overline{\node}')d\mu d\mu\\
&\phantom{=}-\frac{1}{4}\int_{\pointspace^2}\int_{\pointspace^2} g(\node,\node')\kernelf(\node',\node,\overline{\node},\overline{\node}')\hypothesis(\overline{\node},\overline{\node}')d\mu d\mu\\
&\phantom{=}-\frac{1}{4}\int_{\pointspace^2}\int_{\pointspace^2} g(\node,\node')\kernelf(\node,\node',\overline{\node}',\overline{\node})\hypothesis(\overline{\node},\overline{\node}')d\mu d\mu\\
&\phantom{=}+\frac{1}{4}\int_{\pointspace^2}\int_{\pointspace^2} g(\node,\node')\kernelf(\node',\node,\overline{\node}',\overline{\node})\hypothesis(\overline{\node},\overline{\node}')d\mu d\mu\\
&=\frac{1}{4}\left\langle\int_{\pointspace^2} \kernelf_{(\overline{\node},\overline{\node}')}\hypothesis(\overline{\node},\overline{\node}')d\mu,\int_{\pointspace^2} \kernelf_{(\node,\node')}g(\node,\node')d\mu\right\rangle\\
&\phantom{=}-\frac{1}{4}\left\langle\int_{\pointspace^2} \kernelf_{(\overline{\node},\overline{\node}')}\hypothesis(\overline{\node},\overline{\node}')d\mu,\int_{\pointspace^2} \kernelf_{(\node',\node)}g(\node',\node)d\mu\right\rangle\\
&\phantom{=}-\frac{1}{4}\left\langle\int_{\pointspace^2} \kernelf_{(\overline{\node}',\overline{\node})}\hypothesis(\overline{\node}',\overline{\node})d\mu,\int_{\pointspace^2} \kernelf_{(\node,\node')}g(\node,\node')d\mu\right\rangle\\
&\phantom{=}+\frac{1}{4}\left\langle\int_{\pointspace^2} \kernelf_{(\overline{\node}',\overline{\node})}\hypothesis(\overline{\node}',\overline{\node})d\mu,\int_{\pointspace^2} \kernelf_{(\node',\node)}g(\node',\node)d\mu\right\rangle\\
&=\left\langle\int_{\pointspace^2} \Phi^A_{(\overline{\node},\overline{\node}')}\hypothesis(\overline{\node},\overline{\node}')d\mu,\int_{\pointspace^2} \Phi^A_{(\node,\node')}g(\node,\node')d\mu\right\rangle\;,
\end{align*}
where $\Phi^A_{(\overline{\node},\overline{\node}')}=\frac{1}{2}\left(\kernelf_{(\overline{\node},\overline{\node}')}-\kernelf_{(\overline{\node}',\overline{\node})}\right)$. Then,
\begin{align*}
&\int_{\pointspace^2} \Phi^A_{(\overline{\node},\overline{\node}')}\hypothesis(\overline{\node},\overline{\node}')d\mu(\overline{\node},\overline{\node}')\\
&=\frac{1}{2}\int_{\pointspace^2} \kernelf_{(\overline{\node},\overline{\node}')}\hypothesis(\overline{\node},\overline{\node}')d\mu(\overline{\node},\overline{\node}')
-\frac{1}{2}\int_{\pointspace^2} \kernelf_{(\overline{\node}',\overline{\node})}\hypothesis(\overline{\node},\overline{\node}')d\mu(\overline{\node},\overline{\node}')\\
&=\frac{1}{2}\int_{\pointspace^2} \kernelf_{(\overline{\node},\overline{\node}')}\hypothesis(\overline{\node},\overline{\node}')\mu(\overline{\node},\overline{\node}')d(\overline{\node},\overline{\node}')
-\frac{1}{2}\int_{\pointspace^2} \kernelf_{(\overline{\node},\overline{\node}')}\hypothesis(\overline{\node}',\overline{\node})\mu(\overline{\node}',\overline{\node})d(\overline{\node},\overline{\node}')\\
&=\frac{1}{2}\int_{\pointspace^2} \kernelf_{(\overline{\node},\overline{\node}')}\left(\hypothesis(\overline{\node},\overline{\node}')-\frac{\mu(\overline{\node}',\overline{\node})}{\mu(\overline{\node},\overline{\node}')}\hypothesis(\overline{\node}',\overline{\node})\right)\mu(\overline{\node},\overline{\node}')d(\overline{\node},\overline{\node}')\\
&=\int_{\pointspace^2} \kernelf_{(\overline{\node},\overline{\node}')}\left(\asymm^\mu\hypothesis(\overline{\node},\overline{\node}')\right)d\mu(\overline{\node},\overline{\node}')\\
&=\lmutorkhs_{\kernelf}(\asymm^\mu\hypothesis)
\end{align*}
Accordingly, we observe that:
\begin{align*}
\langle\koper_{\kernelf^A}\hypothesis, g\rangle_{L^2(\pointspace^2, \mu)}
=&\langle\lmutorkhs_{\kernelf}\asymm^\mu\hypothesis,\lmutorkhs_{\kernelf}\asymm^\mu g\rangle_\hypspace\\
=&\langle\asymm^\mu\hypothesis,\lmutorkhs_{\kernelf}^*\lmutorkhs_{\kernelf}\asymm^\mu g\rangle_{L^2(\pointspace^2, \mu)}\\
=&\langle\hypothesis,{\asymm^\mu}^*\lmutorkhs_{\kernelf}^*\lmutorkhs_{\kernelf}\asymm^\mu g\rangle_{L^2(\pointspace^2, \mu)}\\
=&\langle\hypothesis,{\asymm^\mu}^*\koper_{\kernelf}\asymm^\mu g\rangle_{L^2(\pointspace^2, \mu)}\;,
\end{align*}
% \begin{align*}
% \langle\lmutorkhs_{\kernelf^P}\hypothesis,\lmutorkhs_{\kernelf^P} g\rangle_{\hypspace(\kernelf^P)}
% =&\langle\lmutorkhs_{\kernelf}\shufflem^\mu\hypothesis,\lmutorkhs_{\kernelf}\shufflem^\mu g\rangle_{\hypspace(\kernelf)}\\
% =&\langle\shufflem^\mu\hypothesis,\lmutorkhs_{\kernelf}^*\lmutorkhs_{\kernelf}\shufflem^\mu g\rangle_{L^2(\pointspace^2, \mu)}\\
% =&\langle\hypothesis,{\shufflem^\mu}^*\lmutorkhs_{\kernelf}^*\lmutorkhs_{\kernelf}\shufflem^\mu g\rangle_{L^2(\pointspace^2, \mu)}\\
% =&\langle\hypothesis,{\shufflem^\mu}^*\koper_{\kernelf}\shufflem^\mu g\rangle_{L^2(\pointspace^2, \mu)}\;,
% \end{align*}
that is, the integral operator of the anti-symmetric kernel is $\koper_{\kernelf^A}={\asymm^\mu}^*\koper_{\kernelf}\asymm^\mu$.

The integral operators of the other kernels can be constructed analogously via the feature mappings: %$(\lmutorkhs_{\kernelf^S}\hypothesis)(\node,\node')=$
\begin{align*}
\Phi_{\kernelf^P}\hypothesis=&\lmutorkhs_{\kernelf}(\shufflem^\mu\hypothesis)\\
\Phi_{\kernelf^S}\hypothesis=&\lmutorkhs_{\kernelf}(\symm^\mu\hypothesis)\\
\Phi_{\kernelf^{PI}}\hypothesis=&\lmutorkhs_{\kernelf}
\left(\left(\begin{array}{c}
\bm{I}\\
\shufflem^\mu
\end{array}
\right)
\hypothesis\right)\;,
%\int_{\pointspace^2} \kernelf^S_{(\overline{\node},\overline{\node}')}\hypothesis(\overline{\node},\overline{\node}')d\mu(\overline{\node},\overline{\node}')
\end{align*}
where $\left(\begin{array}{c}
\bm{I}\\
\shufflem^\mu
\end{array}
\right)$ is the operator obtained by stacking the operators $\bm{I}$ and $\shufflem^\mu$.

Finally, it is straightforward to check that $\symm^\mu$ and $\asymm^\mu$ are projections due to their idempotence, that is, $\symm^\mu\symm^\mu=\symm^\mu$ and $\asymm^\mu\asymm^\mu=\asymm^\mu$.
%Further, the integral operator of the permutation invariant kernel $\kernelf^{PI}$ is $\symm^\mu\koper_{\kernelf}\symm^\mu g+\asymm^\mu\koper_{\kernelf}\asymm^\mu$.
% \begin{align*}
% \langle\lmutorkhs_{\kernelf^{PI}}\hypothesis,\lmutorkhs_{\kernelf^{PI}} g\rangle_\hypspace
% =&\left\langle\left(\begin{array}{c}\lmutorkhs_{\kernelf}\symm^\mu\\\lmutorkhs_{\kernelf}\asymm^\mu\end{array}\right)\hypothesis,\left(\begin{array}{c}\lmutorkhs_{\kernelf}\symm^\mu\\\lmutorkhs_{\kernelf}\asymm^\mu\end{array}\right) g\right\rangle_\hypspace\\
% =&\langle\hypothesis,\symm^\mu\koper_{\kernelf}\symm^\mu g+\asymm^\mu\koper_{\kernelf}\asymm^\mu g\rangle_{L^2(\pointspace^2, \mu)}\\
% %=&\langle\hypothesis,\koper_{\kernelf^A} g\rangle_{L^2(\pointspace^2, \mu)}\;,
% \end{align*}
\end{proof}

\subsection{Proof of Theorem~\ref{generalantisymmetrictheorem}}\label{generalantisymmetricproof}
\begin{proof}
We first consider the RKHS of the permutation invariant kernel $\kernelf^{PI}(\node,\node',\overline{\node},\overline{\node}')$ given in Definition \ref{strangekdef}. According to the theorem concerning sums of reproducing kernels by \citet{aronszajn1950}, the RKHS of the permutation invariant kernel $\kernelf^{PI}$ can be written as the following space of functions:
\begin{eqnarray*}%\label{directsumrkhs}
\hypspace\left(\kernelf^{PI}\right)&=&\hypspace\left(\kernelf+\kernelf^P\right)\\
&=&\left\{\predfun_1 + \predfun_2 : \predfun_1 \in \hypspace\left(\kernelf\right) , \predfun_2 \in \hypspace\left(\kernelf^P\right)\right\}.
\end{eqnarray*}
This, together with the assumption $\funset\subseteq\hypspace\left(\kernelf\right)$, implies
\begin{equation}\label{dpkresult}
\funset\subseteq\hypspace\left(\kernelf^{PI}\right).
\end{equation}
Let $\epsilon>0$ and $\asymfun\in\mathcal{A}$ be an arbitrary function for which $\asymfun(\node, \node')=\arbfun(\node, \node')-\arbfun(\node',\node)$, where $\arbfun\in\funset$. According to (\ref{dpkresult}), we can select a set of pairs $\{(\overline{\node}_i,\overline{\node}'_i)\}_{i=1}^\tsize$ and real numbers $\{\alpha_i\}_{i=1}^\tsize$, such that the function
\begin{equation*}
\approxfun(\node,\node')=\sum_{i=1}^\tsize\alpha_i\kernelf^{PI}(\node,\node',\overline{\node}_i,\overline{\node}'_i)
\end{equation*}
belonging to the RKHS of the kernel $\kernelf^{PI}$ fulfills
\begin{equation}\label{tempapproxone}
\max_{(\node,\node')\in \pointspace^2}\left\{\left\arrowvert \arbfun(\node,\node')-\approxfun(\node,\node')\right\arrowvert\right\}\leq\frac{1}{2}\epsilon \,.
\end{equation}
Let
\[
\hypothesis(\node,\node')=\approxfun(\node,\node') -\approxfun(\node',\node).
\]
It follows from (\ref{tempapproxone}) that
\begin{equation*}
%\max_{(\node,\node')\in \pointspace^2}\left\{\left\arrowvert \asymfun(\node,\node')-(\approxfun(\node,\node') -\approxfun(\node',\node))\right\arrowvert\right\}\leq\epsilon \,.
\max_{(\node,\node')\in \pointspace^2}\left\{\left\arrowvert \asymfun(\node,\node')-\hypothesis(\node',\node)\right\arrowvert\right\}\leq\epsilon \,.
\end{equation*}
We observe that $\hypothesis$ can be written in terms of the kernel $\kernelf^A$ as
\begin{align*}
\hypothesis(\node,\node')&=\sum_{i=1}^\tsize\alpha_i\kernelf^{PI}(\node,\node',\overline{\node}_i,\overline{\node}'_i)\\
&\phantom{=}-\sum_{i=1}^\tsize\alpha_i\kernelf^{PI}(\node',\node,\overline{\node}_i,\overline{\node}'_i)\\
&=\sum_{i=1}^\tsize\alpha_i\kernelf^A(\node,\node',\overline{\node}_i,\overline{\node}'_i)
\end{align*}
which proves the claim for the anti-symmetric kernels. The proof for the symmetric ones is analogous.
\end{proof}

\subsection{Proof of Lemma~\ref{expressionlemma}}\label{expressionproof}

\begin{proof}
%The bias caused by regularization can be calculated by considering the population case. Namely, 
Starting from the form given by \citet{Cucker2002foundations} and applying the Sherman-Morrison-Woodbury fomula for operators \citep{Deng2011woodbury}, we get
\begin{align*}
\predfun^\regparam&=\lmutorkhs_\kernelf^*\hypothesis^\regparam\\
&=\lmutorkhs_\kernelf^*(\lmutorkhs_\kernelf\lmutorkhs_\kernelf^*+\regparam\idmatrix)^{-1}\lmutorkhs_\kernelf\predfun\\
&=\lmutorkhs_\kernelf^*\lmutorkhs_\kernelf(\lmutorkhs_\kernelf^*\lmutorkhs_\kernelf+\regparam\idmatrix)^{-1}\predfun\\
&=\koper_\kernelf(\koper_\kernelf+\regparam\idmatrix)^{-1}\predfun\\
&=\bm{V}\bm{\Lambda}(\bm{\Lambda}+\regparam\idmatrix)^{-1}\bm{V}^*\predfun\;.
\end{align*}
The bias caused by regularization is the squared error between the regression function and $\lmutorkhs_\kernelf^*\hypothesis^\regparam$
%, which can be formalized in terms of the $L^2(\anyspace, \mu)$ operators as follows:
\begin{align*}
\epsilon_{rg}(\predfun,\koper,\regparam)&=\int_\anyspace\left(\predfun(x)-\lmutorkhs_\kernelf^*\hypothesis^\regparam(x)\right)^2d\mu\\
&=\left\langle\predfun-\predfun^\regparam,\predfun-\predfun^\regparam\right\rangle\\
&=\left\langle\predfun-\koper(\koper+\regparam\idmatrix)^{-1}\predfun,\predfun-\koper(\koper+\regparam\idmatrix)^{-1}\predfun\right\rangle\\
&=\left\langle\predfun,\bm{V}\left(\idmatrix-2\bm{\Lambda}(\bm{\Lambda}+\regparam\idmatrix)^{-1}+\bm{\Lambda}^{2}(\bm{\Lambda}+\regparam\idmatrix)^{-2}\right)\bm{V}^*\predfun\right\rangle\\
&=\left\langle\predfun^,\bm{V}\left(\idmatrix-\bm{\Lambda}(\bm{\Lambda}+\regparam\idmatrix)^{-1}\right)^{2}\bm{V}^*\predfun\right\rangle\\
&=\regparam^{2}\left\langle\predfun,\bm{V}\left(\bm{\Lambda}+\regparam\idmatrix\right)^{-2}\bm{V}^*\predfun\right\rangle\\
&=\regparam^{2}\left\langle\predfun,\left(\koper+\regparam\idmatrix\right)^{-2}\predfun\right\rangle\;,
\end{align*}
where the products are in $L^2(\anyspace, \mu)$.
\end{proof}

\subsection{Proof of Theorem~\ref{regerrpropo}}\label{regerrproof}
\begin{proof}
Let the regression function be symmetric, that is, it can be written as $\predfun=\symm\predfun$. Then,
\begin{align*}
\predfun^\regparam_{\kernelf^{PI}}&=\bm{V}\bm{\Lambda}(\bm{\Lambda}+\regparam\idmatrix)^{-1}\bm{V}^*\predfun\\
&=\bm{V}\bm{\Lambda}(\bm{\Lambda}+\regparam\idmatrix)^{-1}\bm{V}^*\symm\predfun\\
&=\bm{V}\bm{\Lambda}_{\kernelf^S}(\bm{\Lambda}_{\kernelf^S}+\regparam\idmatrix)^{-1}\bm{V}^*\predfun\\
&=\predfun^\regparam_{\kernelf^{S}}\;,
\end{align*}
where the second last inequality is due to the
and hence also the bias caused by regularization is the same for the kernels $\kernelf^{PI}$ and $\kernelf^{S}$. The proof is analogous for the anti-symmetric case.

Let $\anymatrix$ be an operator for which $0<\alpha\idmatrix\leq\anymatrix\leq\beta\idmatrix$, where $\alpha$ and $\beta$ are, respectively, the smallest and largest eigenvalues of $\koper+\regparam\idmatrix$. We first recollect some matrix inequalities we use in the proof.
%  The Kantorovich constant for $\anymatrix$ is
% \[
% \kappa(\anymatrix)=\frac{(\alpha+\beta)^2}{4\alpha\beta}\;.
% \]

Choi's inequality and Kadison's inequality (see e.g. \citet{choi1974schwarz}) indiate that if $\anymatrix>0$ and $\Psi$ is positive and unital linear map, then
\begin{align}
\Psi(\anymatrix^{-1}) &\geq \Psi(\anymatrix)^{-1}\label{choi}\\
\Psi(\anymatrix^{2}) &\geq \Psi(\anymatrix)^{2}\label{kadison}\;.
\end{align}
Let $0 < \alpha \leq \anymatrix \leq \beta$ and $\Psi$ be positive unital linear map, \citet{Marshall1990matrixversions} proved the following operator Kantorovich type of inequality:
\begin{align}\label{marshall}
\Psi(\anymatrix^{-1}) \leq\frac{(\alpha+\beta)^2}{4\alpha\beta}\Psi(\anymatrix)^{-1}\;.
\end{align}
%where the so-called 

According to the L{\"o}wner-Heinz Theorem (see e.g. \citet{carlen2010course}), if $\anymatrix$ and $\othermatrix$ are operators and $\anymatrix\geq\othermatrix\geq0$, then
% $-(\cdot)^{-1}$ is operator monotone, that is,
matrix inversion reverses the positive-definite order, that is,
\begin{align}\label{lownerheinzinv}
\anymatrix^{-1}\leq\othermatrix^{-1}
\end{align}
Further, \citet{Fujii1997operatorineq} proved the following Kantorovich type of inequality:
\begin{align}\label{fujii}
\frac{(\alpha+\beta)^2}{4\alpha\beta}\anymatrix^{2} \geq\othermatrix^{2}
\end{align}

Armed with the above matrix inequalities, we get the following combined results:
\begin{align*}
\Psi(\anymatrix)^{-2} &\geq \Psi(\anymatrix^{2})^{-1}\\
& \geq \frac{4\alpha^{2}\beta^{2}}{(\alpha^{2}+\beta^{2})^2}\Psi(\anymatrix^{-2})\;,
\end{align*}
where the first inequality is due to combining (\ref{kadison}) with (\ref{lownerheinzinv}), and the second inequality is due to (\ref{marshall}).
\begin{align*}
\Psi(\anymatrix)^{-2}
&\leq \frac{(\alpha^{-1}+\beta^{-1})^2}{4\alpha^{-1}\beta^{-1}}\Psi(\anymatrix^{-1})^{2}\\
&= \frac{(\alpha+\beta)^2}{4\alpha\beta}\Psi(\anymatrix^{-1})^{2}\\
& \leq\frac{(\alpha+\beta)^2}{4\alpha\beta}\Psi(\anymatrix^{-2})\;,
\end{align*}
where the first inequality is due to combining the Choi's inequality (\ref{choi}) with the inequality (\ref{fujii}), and the second inequality is due to the Kadison's inequality (\ref{kadison}).

Let $\Psi(\anymatrix)=\symm\anymatrix\symm+\asymm\anymatrix\asymm$, which is a unital, positive and linear mapping on $\mathcal{B}(L^2(\anyspace, \mu))$. Then, we have $\koper_{\kernelf^{PI}}+\regparam\idmatrix=\Psi(\koper_{\kernelf}+\regparam\idmatrix)$.
Combining the above results, we get
\begin{align*}
%\frac{4\alpha^2\beta^2}{(\alpha^2+\beta^2)^2}\Psi\left((\koper_{\kernelf^{PI}}+\regparam\idmatrix)^{-2}\right)
\epsilon_{rg}(\predfun,\koper_{\kernelf^{PI}},\regparam)
&=\regparam^2\left\langle\predfun,(\koper_{\kernelf^{PI}}+\regparam\idmatrix)^{-2}\predfun\right\rangle\\
&=\regparam^2\left\langle\predfun,\Psi(\koper_{\kernelf}+\regparam\idmatrix)^{-2}\predfun\right\rangle\\
&\leq\regparam^2\frac{(\alpha+\beta)^2}{4\alpha\beta}\left\langle\predfun,\Psi\left((\koper_{\kernelf}+\regparam\idmatrix)^{-2}\right)\predfun\right\rangle\\
&=\regparam^2\frac{(\alpha+\beta)^2}{4\alpha\beta}\left\langle\symm\predfun,\Psi\left((\koper_{\kernelf}+\regparam\idmatrix)^{-2}\right)\symm\predfun\right\rangle\\
&=\regparam^2\frac{(\alpha+\beta)^2}{4\alpha\beta}\left\langle\predfun,(\koper_{\kernelf}+\regparam\idmatrix)^{-2}\predfun\right\rangle\;,
\end{align*}
where the second last equality is due to the assumption of the regression function being symmetric. The lower bound can be shown analogously.

The limit of the smallest eigenvalue of $\koper$ is $0$, and hence that of $\koper+\regparam\idmatrix$ is $\regparam$. Moreover, due to $\kernelf$ $\max_{(\node,\node')\in\pointspace^2}\kernelf(\node,\node',\node,\node')=1$, the largest eigenvalue of $\koper+\regparam\idmatrix$ is at most $1+\regparam$. The claimed relationship is obtained by substituting $\regparam$ and $1+\regparam$ to $\alpha$ and $\beta$.
\end{proof}

\bibliographystyle{unsrtnat}
\bibliography{myBibliography}

\end{document}